\documentclass[11pt,a4paper]{article}
\usepackage[hyperref]{emnlp2020}
\usepackage{times}
\usepackage{latexsym}

\usepackage{microtype}
\usepackage{amsthm,amsmath,amssymb,amsfonts,bbm}
\usepackage{url}
\usepackage[ruled,linesnumbered,noend]{algorithm2e}
\usepackage{graphicx}
\usepackage{xspace}
\usepackage{xcolor,colortbl}
\usepackage{float}

\newcommand{\paran}[1]{\left( #1 \right)}

\newcommand{\euc}[1]{\| #1 \|_2}
\newcommand{\maha}[1]{\| #1 \|_{\textsf{RM}}}
\usepackage{balance}

\newtheorem{definition}{Definition}
\newtheorem{lemma}{Lemma}
\newtheorem{theorem}{Theorem}

\aclfinalcopy
\title{A Differentially Private Text Perturbation Method \\Using a Regularized Mahalanobis Metric}

\author{Zekun Xu \\
  Amazon Alexa \\
  Seattle, WA USA \\
  \texttt{\normalsize zeku@amazon.com} \\\And
  Abhinav Aggarwal \\
  Amazon Alexa\\
  Seattle, WA USA \\
  \texttt{\normalsize aggabhin@amazon.com} \\\And
  Oluwaseyi Feyisetan \\
  Amazon Alexa \\
  Seattle, WA USA\\
  \texttt{\normalsize sey@amazon.com} \\\And
  Nathanael Teissier \\
  Amazon Alexa \\
  Arlington, VA USA\\
  \texttt{\normalsize natteis@amazon.com}}

\date{}

\begin{document}
\maketitle

\begin{abstract}
Balancing the privacy-utility tradeoff is a crucial requirement of many practical machine learning systems that deal with sensitive customer data. A popular approach for privacy-preserving text analysis is noise injection, in which text data is first mapped into a continuous embedding space, perturbed by sampling a spherical noise from an appropriate distribution, and then projected back to the discrete vocabulary space. While this allows the perturbation to admit the required metric differential privacy, often the utility of downstream tasks modeled on this perturbed data is low because the spherical noise does not account for the variability in the density around different words in the embedding space. In particular, words in a sparse region are likely unchanged even when the noise scale is large. 

In this paper, we propose a text perturbation mechanism based on a carefully designed regularized variant of the Mahalanobis metric to overcome this problem. For any given noise scale, this metric adds an elliptical noise to account for the covariance structure in the embedding space. This heterogeneity in the noise scale along different directions helps ensure that the words in the sparse region have sufficient likelihood of replacement without sacrificing the overall utility. We provide a text-perturbation algorithm based on this metric and formally prove its privacy guarantees. Additionally, we empirically show that our mechanism improves the privacy statistics to achieve the same level of utility as compared to the state-of-the-art Laplace mechanism.
\end{abstract}

\section{Introduction}

Machine learning has been successfully utilized in a wide variety of real world applications including information retrieval, computer graphics, speech recognition, and text mining. 
Technology companies like Amazon, Google, and Microsoft already provide MLaaS (Machine Learning as a Service), where customers can input their datasets for model training and receive black-box prediction results as output. 
However, those datasets may contain personal and potentially sensitive information, which can be exploited to identify the individuals in the datasets, even if it has been anonymized~\cite{sweeney97,narayanan08}.
Removing personally identifiable information is often inadequate, since having access to the summary statistics on the dataset has been shown to be sufficient to infer individual's membership in the dataset
with high probability \cite{homer08,sankararaman09,dwork2015d}.
Moreover, machine learning models themselves can reveal information on the training data.
In particular, sophisticated deep neural networks for natural language processing tasks like next word prediction or neural machine translation, often tend to memorize their training data, which makes them vulnerable to leaking information about their training data~\cite{shokri17,salem2018ml}. 

To provide a quantifiable privacy guarantee against such information leakage, Differential Privacy (DP) has been adopted as a standard framework for privacy-preserving analysis in statistical databases~\cite{dwork2006calibrating,dwork2008differential,dwork2014algorithmic}. Intuitively, a randomized algorithm is differentially private if the output distributions from two neighboring databases are indistinguishable. 
However, a direct application of DP to text analysis can be too restrictive because it requires a lower bound on the probability of any word to be replaced by any other word in the vocabulary.

Metric differential privacy arises as a generalization of local differential privacy \cite{kasiviswanathan2011can}, 
which originated in protecting location privacy such that 
locations near the user's location are assigned with higher probability while those far away are given negligible probability \cite{andres2013geo,chatzikokolakis2013broadening}.
In the context of privacy-preserving text analysis, metric differential privacy implies that
the indistinguishability of the output distributions of any two words in the vocabulary is scaled by their distance, where the distance metrics used in the literature include Hamming distance (reduced to DP), Manhattan distance \cite{chatzikokolakis2015constructing}, Euclidean distance \cite{chatzikokolakis2013broadening,fernandes2019generalised,feyisetan2020privacy}, Chebyshev distance \cite{wagner2018technical}, hyperbolic distance \cite{feyisetan2019leveraging}.

In this paper, we propose a  novel privacy-preserving text perturbation method by adding an elliptical noise to word embeddings in the Euclidean space, where the scale of the noise is calibrated by the regularized Mahalanobis norm (formally defined in Section 3).
We compare our method to the existing multivariate Laplace mechanism for privacy-preserving text analysis in the Euclidean space \cite{fernandes2019generalised,feyisetan2020privacy}.
In both papers, text perturbation is implemented by adding a spherical noise sampled from multivariate Laplace distribution to the original word embedding. 
However, the spherical noise does not account for the structure in the embedding space.
In particular, words in a sparse region are likely unchanged even when the scale of noise is large.
This can potentially result in severe privacy breach when sensitive words do not get perturbed. 
To increase the substitution probability of words in sparse regions, the scale of noise has to be large in the multivariate Laplace mechanism, which will hurt the downstream machine learning utility. 

We address this problem by adding an elliptical noise to word embeddings according to the covariance structure in the embedding space.
The intuition is that given a fixed scale of noise, we want to stretch the noise equidistant contour in the direction so that the substitution probability of words in the sparse region is increased on average.
Intuitively, this direction is the one that explains the largest variability in the word embedding vectors in the vocabulary.
We prove the theoretical metric differential privacy guarantee of the proposed method.
Furthermore, we use empirical analysis to show that the proposed method significantly improves the privacy statistics while achieves the same level of utility as compared to the multivariate Laplace mechanism.
Our main contributions are as follows:
\begin{itemize}
    \item We develop a novel Mahalanobis mechanism for differentially private text perturbation, which calibrates the elliptical noise by accounting for the covariance structure in the word embedding space.
    \item A theoretical metric differential privacy proof is provided for the proposed method. 
    \item We compare the privacy statistics and utility results between our method and the multivariate Laplace Mechanism through experiments, which 
    demonstrates that our method has significantly better privacy statistics while preserving the same level of utility.
\end{itemize}

\section{Related Works}

Privacy-preserving text analysis is a well-studied problem in the literature \cite{hill2016effectiveness}. One of the common approaches is to identify sensitive terms (like personally identifiable information) in a document and replace them with some more general terms \cite{cumby2011machine,anandan2012t,sanchez2016c}.
Another line of research achieves text redaction by injecting additional words into the original text without detecting sensitive entities \cite{domingo2009h,pang2010embellishing,sanchez2013knowledge}.
However, those methods are shown to be vulnerable to reidentification attacks \cite{petit2015peas}.
 
In order to provide a quantifiable theoretical privacy guarantee, the differential privacy framework \cite{dwork2008differential} has been used for privacy-preserving text analysis. 
In the DPText model \cite{beigi2019not}, an element-wise univariate Laplace noise is added to the pre-trained auto-encoders to provide privacy for text representations.
Another approach for privacy-preserving text perturbation is in the metric differential privacy framework \cite{andres2013geo,chatzikokolakis2013broadening}, an extended notion of local differential privacy \cite{kasiviswanathan2011can}, which adds noise to the pre-trained word embeddings.
Metric differential privacy requires that the indistinguishability of the output distributions of any two words in the vocabulary be scaled by their distance, which reduces to differential privacy when Hamming distance is used \cite{chatzikokolakis2013broadening}.
A hyperbolic distance metric \cite{feyisetan2019leveraging} was proposed  to provide privacy by perturbing vector representations of words, but it requires specialized training of word embeddings in the high-dimensional hyperbolic space.
For the word embeddings in the Euclidean space \cite{fernandes2019generalised,feyisetan2020privacy}, text perturbation is implemented by sampling independent spherical noise from multivariate Laplace distributions.
The former work \cite{fernandes2019generalised} subsequently used an Earth mover's metric to derive a Bag-of-Words representation on the text, whereas the latter \cite{feyisetan2020privacy} directly worked on the word-level embeddings. 
Since we work with word embeddings in the Euclidean space, we compare our method to the multivariate Laplace mechanism for text perturbation in those two papers \cite{fernandes2019generalised,feyisetan2020privacy}.

Mahalanobis distance has been used as a sensitivity metric for differential privacy in functional data \cite{hall2012new,hall2013differential} and differentially private outlier analysis \cite{okada2015differentially}.
Outside the realm of text analysis, Mahalanobis distance is a common tool in cluster analysis, pattern recognition, and anomaly detection \cite{de2000mahalanobis,xiang2008learning,warren2011use,zhao2015mahalanobis,zhang2015low}.
 
\section{Methodology}

We begin by formally defining Euclidean norm and the regularized Mahalanobis norm.

\begin{definition}[Euclidean Norm]\label{def1}
For any vector $x\in\mathbb{R}^m$, its Euclidean norm is:
$\euc{x}=\sqrt{x^\intercal x}.$
\end{definition}

\begin{definition}[Mahalanobis Norm]
For any vector $x\in\mathbb{R}^m$, and a positive definite matrix $\Sigma$, its Mahalanobis norm is:
$$\|x\|_{\textrm{M}}=\sqrt{x^\intercal \Sigma^{-1}x}.$$
\end{definition}

\begin{definition}[Regularized Mahalanobis Norm]
For any vector $x\in\mathbb{R}^m$, $\lambda\in[0,1]$, and a positive definite matrix $\Sigma$, its regularized Mahalanobis norm is:
$$\maha{x}=\sqrt{x^\intercal \{\lambda\Sigma+(1-\lambda)I_m\}^{-1}x}.$$
\end{definition}

From the definitions above, $\lambda$ can be considered as a tuning parameter. When $\lambda=0$, the regularized Mahalanobis norm reduces to the Euclidean norm; when $\lambda=1$, the regularized norm reduces to the Mahalanobis norm \cite{maha36} 

Note that for any $\eta>0$, the trajectory of $\{y\in\mathbb{R}^m: \|y-x\|_2=\eta\}$ is spherical, whereas the trajectory of $\{y\in\mathbb{R}^m:\maha{y-x}=\eta\}$ is elliptical unless $\lambda=0$.
We will exploit this key difference in the geometry of equidistant contour between the Euclidean norm and the regularized Mahalanobis norm to motivate our text perturbation method.
The proposed regularized Mahalanobis norm is a type of shrinkage estimator \cite{daniels2001shrinkage,schafer2005shrinkage,couillet2014large}, which is commonly used to estimate the covariance matrix of high-dimensional vectors so as to ensure the stability of the estimator.
The matrix $\Sigma$ in the regularized Mahalanobis norm controls the direction to which the equidistant contour in the noise distribution is stretched, 
while the parameter $\lambda$ controls the degree of the stretch.

\begin{definition}[Metric Differential Privacy.]\label{def5}
For any $\epsilon>0$, a randomized algorithm $M:\mathcal{X}\to\mathcal{Y}$ satisfies $\epsilon$ $d_{\mathcal{X}}-$privacy if for any $x,x'\in\mathcal{X}$ and $y\in\mathcal{Y}$, the following holds:
$$\frac{\Pr\{M(x)=y\}}{\Pr\{M(x')=y\}}\leq \exp\{\epsilon d(x,x')\}.$$
\end{definition}
\begin{figure}[t]
    \centering
    \includegraphics[width=0.5\textwidth]{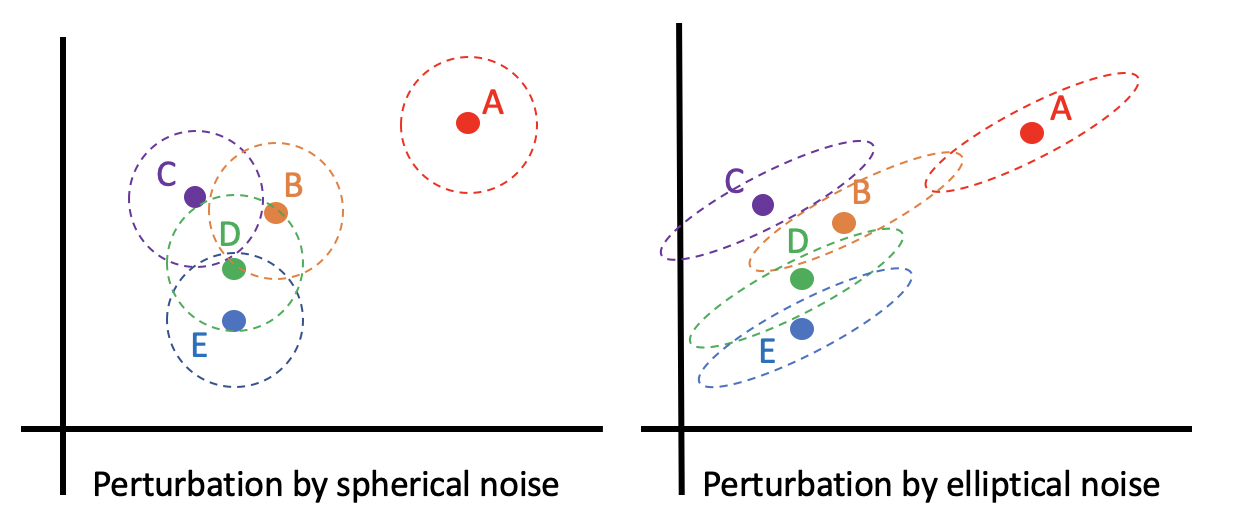}
    \caption{{\bf{Left:}} dotted trajectories represent perturbation by spherical noise in multivariate Laplace mechanism. {\bf{Right:}} dotted trajectories represent perturbation by elliptical noise in proposed Mahalanobis mechanism. The scale of noise is the same in both plots.}
    \label{fig1}
\end{figure}
Metric differential privacy ($d_{\mathcal{X}}-$privacy) originated in privacy-preserving geolocation studies \cite{andres2013geo,chatzikokolakis2013broadening}, where the metric $d$ is Euclidean distance. 
It has been extended to quantifying privacy guarantee in text analysis, which states that for any two words $w$, $w'$ in the vocabulary $\mathcal{W}$, the likelihood ratio of observing any $\hat{w}\in\mathcal{W}$ is bounded by $\epsilon d(w,w')$, where $d(w,w')=\|\phi(w)-\phi(w')\|_2$ for an embedding function $\phi:\mathcal{W}\to\mathbb{R}^m$.

The multivariate Laplace mechanism \cite{fernandes2019generalised,feyisetan2020privacy} perturbs the word embedding $\phi(w)$ by adding a random spherical noise $Z$ sampled from density $f_Z(z)\propto e^{\epsilon\|z\|_2}$, and then find the nearest neighbor in the embedding space as the output of the mechanism.

The left panel in Figure \ref{fig1} illustrates the text perturbation by spherical noise in the multivariate Laplace mechanism. Here $A$ is in the sparse region in the two-dimensional embedding space. 
Given the privacy budget $\epsilon$, $A$ has a small probability of being substituted by other words because its expected perturbation (dotted trajectory) still has itself as the nearest neighbor. 
The right panel in Figure \ref{fig1} shows text perturbation in the proposed Mahalanobis mechanism, where the word embeddings are redacted by an elliptical noise at the same privacy budget $\epsilon$. 
The matrix $\Sigma$ is taken to be the sample covariance matrix of the word embeddings scaled by the mean sample variance, so that the noise contour is stretched toward the direction that explains the largest variability in the embedding space.
The purpose of the scaling step is to ensure that the scale of the elliptical noise is the same as the scale of the spherical noise.
By transforming the spherical noise contour into an elliptical contour, we increase the substitution probability of $A$ in the sparse region, thus improving the privacy guarantee. 
Meanwhile, since the scale of noise does not change, the utility is preserved at the same level.
This is an illustrative example demonstrating the intuition and motivation of the proposed Mahalanobis mechanism.

\begin{algorithm*}[h]
\textbf{Input: }Dimension $m$, a positive definite matrix $\Sigma$, tuning parameter $\lambda\in[0,1]$\\
 \SetAlgoLined
 Sample an $m$-dimensional random vector $N$ from a multivariate normal distribution with mean zero and identity covariance matrix. \\
 Normalize $X=N/\euc{N}$. \\
 Sample $Y$ from a Gamma distribution with shape parameter $m$ and scale parameter $1/\epsilon$.\\
 Return $Z=Y\cdot\{\lambda\Sigma+(1-\lambda) I_m\}^{1/2} X$.
 \caption{Sampling from $f_Z(z)\propto \exp(-\epsilon\maha{z})$}
 \label{alg1}
\end{algorithm*}
  
 \begin{algorithm*}[h]
\textbf{Input: }String $s=w_1w_2\ldots w_n$, privacy parameter $\epsilon>0$, scaled sample covariance matrix $\Sigma$, tuning parameter $\lambda\in[0,1]$\\
 \SetAlgoLined
 \For{$i \in \{1,\ldots,n\}$}{
 	Sample $Z$ from $f_Z(z)\propto\exp(-\epsilon\maha{z})$ using Algorithm 1.\\
 	Obtain the perturbed embedding $\hat{\phi}_i=\phi(w_i)+Z$.\\
 	Replace $w_i$ with $\hat{w}_i=\arg\min_{w\in\mathcal{W}}\euc{\phi(w)-\hat{\phi}_i}$.\\}
 \Return $\tilde{s}=\hat{w}_1\hat{w}_2\dots \hat{w}_n$.
 \caption{The Mahalanobis Mechanism}
 \label{alg2}
\end{algorithm*}

Our proposed Mahalanobis mechanism for text perturbation shares the same general structure as the multivariate Laplace mechanism. The key difference is that the spherical noise $f_Z(z)\propto \exp\paran{\epsilon\|z\|_2}$ in the multivariate Laplace mechanism is replaced by the elliptical noise sampled from density $f_Z(z)\propto\exp(-\epsilon\maha{z})$, which can be efficiently performed via 
Algorithm \ref{alg1}.

 An overview for the proposed Mahalanobis mechanism is presented in Algorithm \ref{alg2}. 
 When $\lambda=0$, the proposed Mahalanobis mechanism reduces to the multivariate Laplace mechanism.
A heuristic method for choosing the tuning parameter $\lambda$ is to find the value of $\lambda$ that maximizes the improvement in the privacy guarantee while maintaining the same level of utility. This can be done through empirical privacy and utility experiments as described in Section 5.  
The input $\Sigma$ in Algorithm \ref{alg2} is computed by scaling the sample covariance matrix of the word embeddings by the mean sample variance so as to guarantee that the trace of $\Sigma$ equals the trace of $I_m$. 
Since $\Sigma$ is a scaled counterpart of the sample covariance matrix, it will stretch the elliptical noise toward the direction with the largest variability in the word embedding space, which maximizes the overall expected probability of words being substituted. 
We remark that in order to maximize the substitution probability for each individual word, a personalized covariance matrix $\Sigma_w$ can be computed in the neighborhood of each word.
This is beyond the scope of this paper and we leave it as future work.

\section{Theoretical Properties}
\begin{lemma}\label{lemma1}
The random variable $Z$ returned from Algorithm \ref{alg1} has a probability density function of the form $f_Z(z)\propto \exp(-\epsilon\maha{z})$.
\end{lemma}

\begin{proof}
Define $U=YX$. Note that conditional on $Y$, $U$ follows a uniform distribution on a sphere with radius $y$ in the $m-$dimensional space, which implies $f_{U|Y}(u|y)\propto 1/y^{m-1}$ when $\sum_{i=1}^m u_i^2=y^2$ and 0 otherwise.
Therefore,
\begin{align*}
    f_U(u)=&\int f_{U|Y}(u|y)f_Y(y) \delta(y=\euc{u}) dy \\
    \propto&\int \frac{1}{y^{m-1}} \frac{\epsilon^m}{\Gamma(m)}y^{m-1}e^{-\epsilon y}\delta(y=\euc{u}) dy \\
   \propto & e^{-\epsilon \sqrt{u^\intercal u}},
\end{align*}
where $\delta(\cdot)$ is the Dirac delta function.
Since $Z=\{\lambda\Sigma+(1-\lambda) I_m\}^{1/2}U$, which is well-defined because $\lambda\in[0,1]$ and $\Sigma$ is positive definite, it follows that:
\begin{align*}
    f_Z(z)\propto & \exp\paran{-\epsilon\sqrt{z^\intercal\{\lambda\Sigma+(1-\lambda)I_m\}^{-1}z}},
\end{align*}
The result in the lemma follows by definition.
\end{proof}

\begin{theorem}\label{theorem1}
For any given $\epsilon > 0$ and $\lambda \in [0,1]$, the Mahalanobis mechanism from Algorithm~\ref{alg2} satisfies $\epsilon$ $d_\chi$-privacy with respect to the Regularized Mahalanobis Norm.
\end{theorem}

\begin{proof}
It suffices to show that for any strings $s=[w_1\ldots w_n]$, $s'=[w_1'\ldots w_n']$, $\hat{s}=[\hat{w}_1\ldots\hat{w}_n]$, $\epsilon>0$, $\lambda\in[0,1]$, 
and positive definite matrix $\Sigma$, 
$$\frac{\Pr\{M(s)=\hat{s}\}}{\Pr\{M(s')=\hat{s}\}}\leq e^{\epsilon\sum_{i=1}^n\maha{\phi(w_i)-\phi(w'_i)}},$$
where $M:\mathcal{W}^n\to\mathcal{W}^n$ is the Mahalanobis mechanism and $\phi: \mathcal{W}\to\mathbb{R}^m$ is the embedding function.

We begin by showing that for any $w,w',\hat{w}\in\mathcal{W}$, it holds that
the probability $\Pr\{M(w)=\hat{w}\}$ is at most $e^{\epsilon \maha{\phi(w)-\phi(w')}}$ times the probability $\Pr\{M(w')=\hat{w}\}$. We define $C_{\hat{w}}=\{v\in\mathbb{R}^m: \euc{v-\phi(\hat{w})} < \min_{w\in\mathcal{W} \backslash \hat{w}}\euc{v-\phi(w)}\}$ be the set of vectors $v$ that are closer to $\hat{w}$ than any other word in the embedding space. Let $Z$ be sampled from $f_Z(z)\propto\exp(-\epsilon\maha{z})$ 
by Algorithm \ref{alg1}, 
\begin{align*}
    \Pr\{M(w)=\hat{w}\} &= \Pr\{\phi(w)+Z\in C_{\hat{w}}\} \\
    = \int_{C_{\hat{w}}} f_Z\paran{v-\phi(w)} dv &= \int_{C_{\hat{w}}} e^{-\epsilon\maha{v-\phi(w)}} dv,
\end{align*} 
where the last step follows from Lemma~\ref{lemma1}. Since $\Sigma$ is positive definite, it admits a spectral decomposition $\Sigma=Q\Lambda Q^\intercal$, where
$Q^\intercal = Q^{-1}$, and $\Lambda$ is a diagonal matrix with positive entries $\xi_1,\ldots,\xi_m$.
Then we can rewrite $\{\lambda\Sigma+(1-\lambda)
    I_m\}^{-1} = Q \Omega Q^\intercal$, where $\Omega^{-1}=\lambda\Lambda+(1-\lambda)I_m$.
Define $\tilde{v}=\Omega^{1/2} Q^\intercal v$ and
$\tilde{\phi}(w)=\Omega^{1/2}Q^\intercal \phi(w)$.
By the triangle inequality, the following hold:
\begin{align*}
    & e^{-\epsilon \sqrt{\{v-\phi(w)\}^\intercal\{\lambda\Sigma+(1-\lambda)
    I_d\}^{-1}\{v-\phi(w)\}}}\\
    =& e^{-\epsilon \sqrt{\{v-\phi(w)\}^\intercal Q\Omega^{1/2}\Omega^{1/2}Q^\intercal \{v-\phi(w)\}}} \\
    =& e^{-\epsilon\euc{\tilde{v}-\tilde{\phi}(w)}}\\
    =& \frac{e^{-\epsilon\euc{\tilde{v}-\tilde{\phi}(w')}}}{e^{-\epsilon\euc{\tilde{v}-\tilde{\phi}(w')}}} e^{-\epsilon\euc{\tilde{v}-\tilde{\phi}(w)}}\\
    \leq & e^{-\epsilon\euc{\tilde{v}-\tilde{\phi}(w')}} e^{\epsilon\euc{\tilde{\phi}(w)-\tilde{\phi}(w')}}\\
    =&
    e^{-\epsilon\maha{v-\phi(w')}}e^{\epsilon\maha{\phi(w)-\phi(w')}}.
\end{align*}
The probability ratio is computed by:
\begin{align*}
    \frac{\Pr\{M(w)=\hat{w}\}}{\Pr\{M(w')=\hat{w}\}} &= \frac{\int_{C_{\hat{w}}} e^{-\epsilon\maha{v-\phi(w)}} dv}{\int_{C_{\hat{w}}} e^{-\epsilon\maha{v-\phi(w')}} dv}\\
    &\le e^{\epsilon\maha{\phi(w)-\phi(w')}}.
\end{align*}
Finally, since each word in the string is processed independently,
\begin{align*}
    \frac{\Pr(M(s)=\hat{s})}{\Pr(M(s')=\hat{s})}&=\prod_{i=1}^n\biggl( \frac{\Pr(M(w_i)=\hat{w}_i)}{\Pr(M(w'_i)=\hat{w}_i)}\biggr)\\
    &\leq\prod_{i=1}^ne^{\epsilon \maha{\phi(w_i)-\phi(w_i')}}\\
    &\leq e^{\epsilon\sum_{i=1}^n\maha{\phi(w_i)-\phi(w'_i)}}. \qedhere
\end{align*}
\end{proof}

Next, we relate the proved theoretical guarantee of the Mahalanobis mechanism to that of the multivariate Laplace mechanism \cite{fernandes2019generalised,feyisetan2020privacy}, which enjoys metrical differential privacy guarantee with respect to the Euclidean metric. The following lemma will help establish our result.

\begin{lemma}
Let $v \in \mathbb{R}^2$ and $m = trace(\Sigma)$. Let $c > 0$ be a lower bound on the smallest eigenvalue of $\Sigma$. Then, the following bounds hold:
\begin{align*}
   \frac{||v||_2}{\sqrt{\lambda m + 1 - \lambda}} \leq \maha{v} \leq \frac{||v||_2}{\sqrt{\lambda c + 1 - \lambda}}.
\end{align*}
\end{lemma}

\begin{proof}
Since $\Sigma$ is positive definite, it admits a spectral decomposition $\Sigma=Q\Lambda Q^\intercal$, where
$Q^\intercal = Q^{-1}$, and $\Lambda$ is a diagonal matrix with eigenvalues $\xi_1,\ldots,\xi_m$. Since by assumption of the minimum eigenvalue greater than $c>0$, and that $trace(\Sigma)=\sum_{i=1}^m\xi_i=m$, we have $\xi_i\in(c,m)$ for $i=1,\ldots,m$.
Then, the eigenvalues for $\{\lambda\Sigma+(1-\lambda)
    I_m\}^{-1}$
are $\frac{1}{\lambda\xi_1+1-\lambda},\ldots,\frac{1}{\lambda\xi_m+1-\lambda}$,
which are between $\frac{\epsilon}{\lambda m + 1 -\lambda}$ and $\frac{\epsilon}{\lambda c + 1 -\lambda}$.
Then for any vector $v\in\mathbb{R}^m$,
\begin{align*}
    \maha{v}^2=&v^\intercal\{\lambda\Sigma+(1-\lambda)
    I_m\}^{-1}v \\
    =&\sum_{i=1}^m \frac{1}{\lambda\xi_i+1-\lambda}(q_i^\intercal v)^2\\
    \leq &\frac{1}{\lambda c+1-\lambda} \sum_{i=1}^m(q_i^\intercal v)^2 = \frac{\|v\|_2^2}{\lambda c+1-\lambda},
\end{align*}
where $\sum_{i=1}^m(q_i^\intercal v)^2=\|Q^\intercal v\|_2^2=\|v\|_2^2$. Similarly, we can show that $\maha{v}^2\geq \frac{1}{\lambda m+1-\lambda}\|v\|_2^2$, so the result follows immediately.
\end{proof}

\begin{figure*}[h!]
    \centering
    \includegraphics[width=0.85\textwidth]{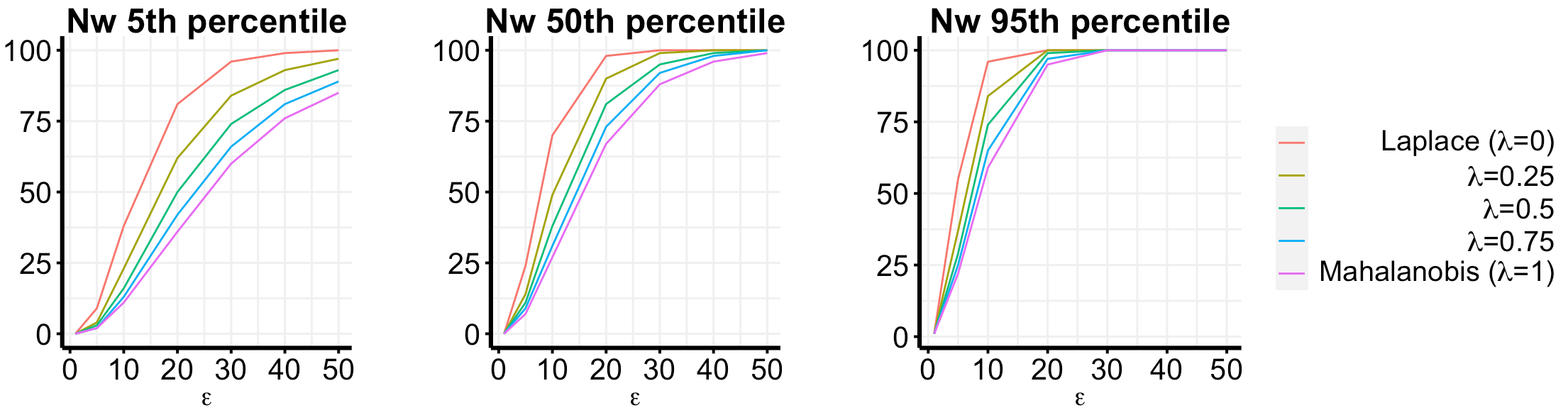}
    \caption{Percentiles for $N_w$ (number of times an input word $w$ does not change) for 300-d FastText embedding over 100 repetitions. The Mahalanobis mechanism has lower values of $N_w$ than the Laplace mechanism.}
    \label{nw_ft}
\end{figure*}

\begin{figure*}[h!]
    \centering
    \includegraphics[width=0.85\textwidth]{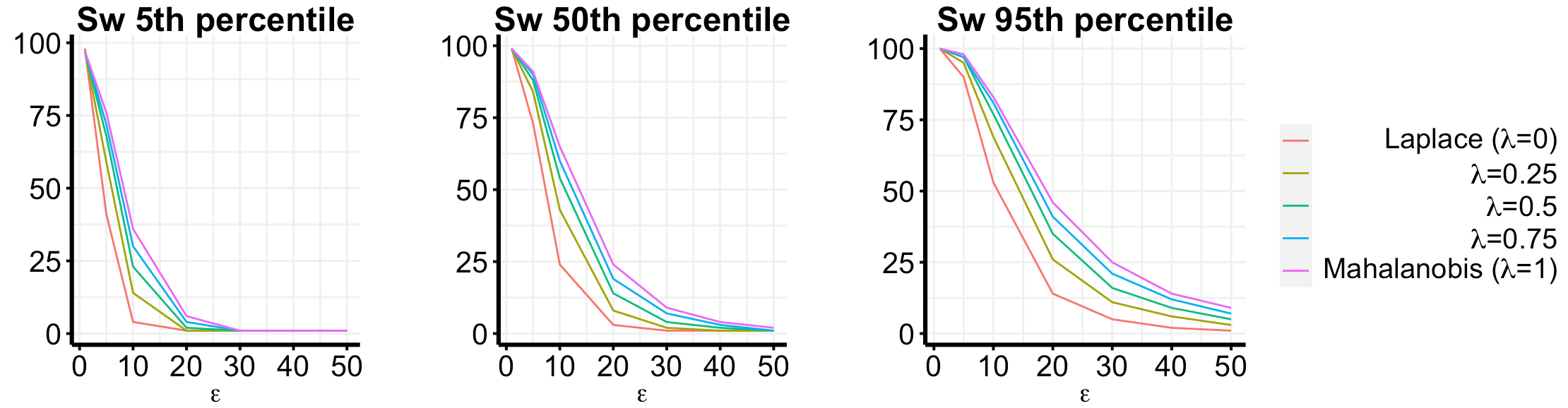}
    \caption{Percentiles for $S_w$ (number of distinct substitutions for an input word $w$) for 300-d FastText embedding over 100 repetitions. The Mahalanobis mechanism has higher values of $S_w$ than the Laplace mechanism.} 
    \label{sw_ft}
\end{figure*}

\begin{table*}[t!]
\begin{tabular}{llllll} 
\hline
 $\epsilon$ & 1 & 5 & 10 & 20 & 40 \\ \hline
 Laplace ($\lambda=0$) & 0.19$\pm$0.46 & 26.89$\pm$14.49 & 68.93$\pm$17.97  & 95.28$\pm$6.46 & 99.85$\pm$0.68\\  
 $\lambda=0.25$ & 0.15$\pm$0.41 & 16.66$\pm$10.53 & 50.31$\pm$18.68  & 86.34$\pm$12.41 & 98.69$\pm$2.81\\
 $\lambda=0.5$ &  0.13$\pm$0.38 & 12.68$\pm$8.60 & 40.40$\pm$17.41 & 78.44$\pm$15.40 & 96.91$\pm$4.88\\  
 $\lambda=0.75$ &  0.12$\pm$0.36 & 10.46$\pm$7.49 & 34.14$\pm$16.03 & 71.81$\pm$17.00 & 94.85$\pm$6.79\\  
 Mahalanobis ($\lambda=1$) & 0.10$\pm$0.35 & 8.97$\pm$6.74 & 29.73$\pm$14.83 & 66.24$\pm$17.90 & 92.80$\pm$8.31\\ \hline
\end{tabular}
\caption{Mean $\pm$ Standard Deviation for $N_w$ for 300-d FastText embedding. For $\epsilon=5,10,20$, the Mahalanobis mechanism has  significantly lower values of mean $N_w$ than the Laplace mechanism, where statistical significance is established by comparing the 95\% confidence intervals in the form of $mean$ $\pm$ $1.96\times std/\sqrt{100}$.}
\label{nw_tab}
\end{table*}

\begin{table*}[t!]
\begin{tabular}{llllll} 
\hline
 $\epsilon$ & 1 & 5 & 10 & 20 & 40 \\ \hline
 Laplace ($\lambda=0$) & 99.23$\pm$0.87 & 70.27$\pm$15.07 & 25.82$\pm$15.06 & 4.53$\pm$4.56 & 1.13$\pm$0.55\\  
 $\lambda=0.25$ & 99.18$\pm$0.90 & 81.42$\pm$11.12 & 42.59$\pm$16.56 & 10.31$\pm$8.08 & 1.93$\pm$1.80\\
 $\lambda=0.5$ & 99.07$\pm$0.96 &  85.69$\pm$9.10 & 52.26$\pm$15.96 &  15.66$\pm$10.18 & 2.98$\pm$2.83\\  
 $\lambda=0.75$ & 98.95$\pm$1.03 & 88.08$\pm$7.96 & 58.66$\pm$15.13 & 20.38$\pm$11.46 & 4.16$\pm$3.75\\  
 Mahalanobis ($\lambda=1$) & 98.88$\pm$1.04 & 89.58$\pm$7.17 & 63.28$\pm$14.25 & 24.60$\pm$12.37 & 5.32$\pm$4.51 \\ \hline
\end{tabular}
\caption{Mean $\pm$ Standard Deviation for $S_w$ for 300-d FastText embedding. For $\epsilon=5,10,20$, the Mahalanobis mechanism has  significantly higher values of mean $S_w$ than the Laplace mechanism, where statistical significance is established by comparing the 95\% confidence intervals in the form of $mean$ $\pm$ $1.96\times std/\sqrt{100}$.}
\label{sw_tab}
\end{table*} 


The lemma below then follows.
\begin{lemma}\label{theorem21}
Assume $trace(\Sigma)=m$ and the minimum eigenvalue of $\Sigma$ is greater than $c$ for some constant $c>0$, then for any $w,w'\in\mathcal{W}$, then
\begin{align*}
    & \exp\paran{\epsilon\maha{\phi(w)-\phi(w')}} \\
    \leq & \exp\paran{\frac{\epsilon}{\sqrt{\lambda c + 1 -\lambda}}\euc{\phi(w)-\phi(w')}}, \;\;\;\textrm{and} \\
    & \exp\paran{\epsilon\maha{\phi(w)-\phi(w')}} \\
    \geq &
    \exp\paran{\frac{\epsilon}{\sqrt{\lambda m + 1 -\lambda}}\euc{\phi(w)-\phi(w')}}.
\end{align*}
\end{lemma}

The fact that the probability ratio as a function of $\maha{\cdot}$ can be sandwiched by lower and upper bounds as functions of $\euc{\cdot}$ shows the noise scale $\epsilon$ is comparable between the Mahalanobis mechanism and multivariate Laplace mechanism. 

\section{Experiments}

We empirically compare the proposed Mahalanobis mechanism and the existing multivariate Laplace mechanism in both privacy experiments and utility experiments on the following two datasets (more details in Appendix \ref{additional}):

\begin{itemize}
    \item {\bf{Twitter Dataset}}. This is a publicly available Kaggle competition dataset (\url{https://www.kaggle.com/c/nlp-getting-started}),
    which contains 7,613 tweets, each with a label indicating whether the tweet describes a disaster event (43\% disaster).  
    \item {\bf{SMSSpam Dataset}}. This is a publicly available dataset from the UCI machine learning repository (\url{https://archive.ics.uci.edu/ml/datasets/SMS+Spam+Collection}), which contains 5,574 (13\% spam) SMS labeled messages  collected for mobile phone spam research \cite{almeida2011contributions}. 
\end{itemize}    

\subsection{Privacy experiments}

\begin{figure*}[h!]
    \centering
    \includegraphics[width=0.9\textwidth]{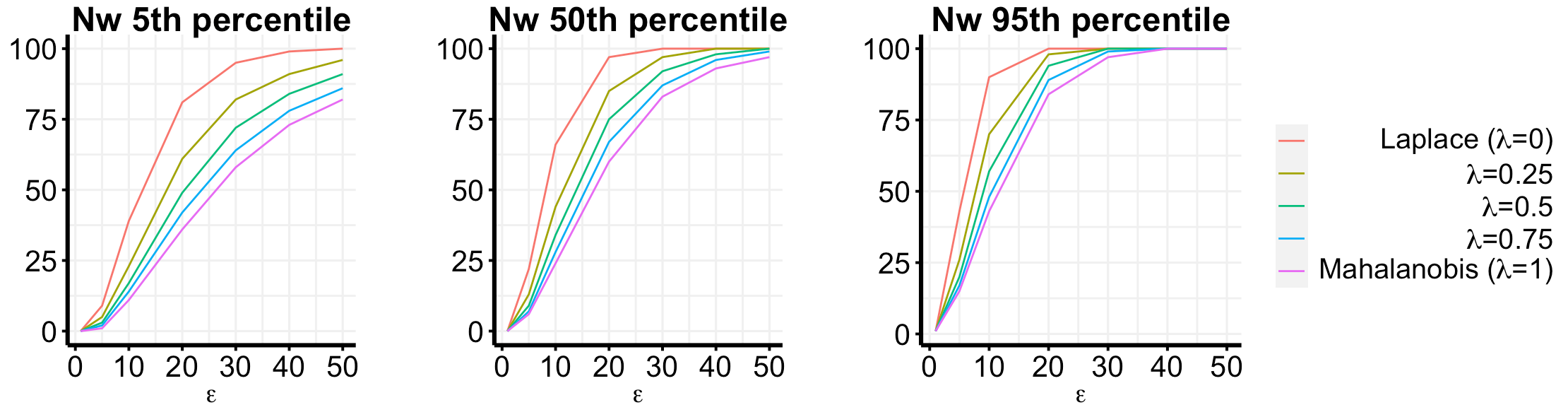}
    \caption{Percentiles for $N_w$ (number of times an input word $w$ does change) for 300-d GloVe embedding over 100 repetitions. The Mahalanobis mechanism has lower values of $N_w$ than the Laplace mechanism.}
    \label{nw_bt}
\end{figure*}

\begin{figure*}[h!]
    \centering
    \includegraphics[width=0.9\textwidth]{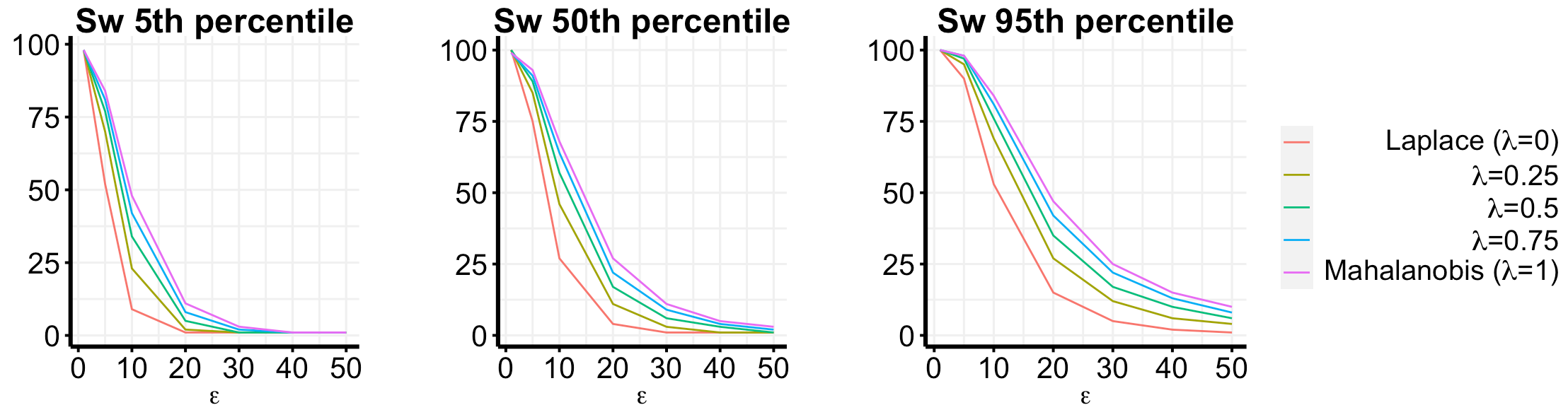}
    \caption{Percentiles for $S_w$ (number of distinct substitutions for an input word $w$) for 300-d GloVe embedding over 100 repetitions. The Mahalanobis mechanism has higher values of $S_w$ than the Laplace mechanism.} 
    \label{sw_bt}
\end{figure*}

\begin{table*}[t!]
\begin{tabular}{llllll} 
\hline
 $\epsilon$ & 1 & 5 & 10 & 20 & 40 \\ \hline
 Laplace ($\lambda=0$) & 0.16$\pm$0.41 & 23.71$\pm$10.69 & 65.29$\pm$15.53 & 94.49$\pm$6.48 & 99.81$\pm$0.74 \\  
 $\lambda=0.25$ & 0.11$\pm$0.34 & 13.82$\pm$6.89 & 44.83$\pm$14.17 & 83.17$\pm$11.61 & 98.22$\pm$3.31 \\
 $\lambda=0.5$ & 0.09$\pm$0.31 & 10.20$\pm$5.57 & 34.91$\pm$12.12 & 73.78$\pm$13.59 & 95.85$\pm$5.56\\  
 $\lambda=0.75$ & 0.08$\pm$0.29 & 8.16$\pm$4.87 & 28.97$\pm$10.72 & 66.29$\pm$14.35 & 93.26$\pm$7.30\\  
 Mahalanobis ($\lambda=1$) & 0.08$\pm$0.28 & 6.81$\pm$4.36 & 24.90$\pm$9.60 & 60.27$\pm$14.54 & 90.60$\pm$8.65\\ \hline
\end{tabular}
\caption{Mean $\pm$ Standard Deviation for $N_w$ for 300-d GloVe embedding. For $\epsilon=5,10,20$, the Mahalanobis mechanism has significantly lower values of mean $N_w$ than the Laplace mechanism, where statistical significance is established by comparing the 95\% confidence intervals in the form of $mean$ $\pm$ $1.96\times std/\sqrt{100}$.}
\label{nw_tab_bt}
\end{table*}

\begin{table*}[t!]
\begin{tabular}{llllll} 
\hline
 $\epsilon$ & 1 & 5 & 10 & 20 & 40 \\ \hline
 Laplace ($\lambda=0$) & 99.46$\pm$0.74 & 73.20$\pm$11.60 & 28.56$\pm$13.61 & 5.18$\pm$4.64 & 1.16$\pm$0.62 \\  
 $\lambda=0.25$ & 99.45$\pm$0.73 & 84.30$\pm$7.62 & 46.08$\pm$13.98 & 12.12$\pm$7.69 & 2.21$\pm$2.02 \\
 $\lambda=0.5$ &  99.38$\pm$0.79 & 88.39$\pm$6.08 & 56.07$\pm$12.92 & 17.96$\pm$9.33 & 3.55$\pm$3.06 \\  
 $\lambda=0.75$ & 99.30$\pm$0.83 & 90.71$\pm$5.27 & 62.69$\pm$11.84 & 23.11$\pm$10.42 & 4.94$\pm$3.90 \\  
 Mahalanobis ($\lambda=1$) & 99.24$\pm$0.86 & 92.17$\pm$4.65 & 67.45$\pm$10.86 & 27.51$\pm$11.12 & 6.39$\pm$4.57\\ \hline
\end{tabular}
\caption{Mean $\pm$ Standard Deviation for $S_w$ for 300-d GloVe embedding. For $\epsilon=5,10,20$, the Mahalanobis mechanism has  significantly higher values of mean $S_w$ than the Laplace mechanism, where statistical significance is established by comparing the 95\% confidence intervals in the form of $mean$ $\pm$ $1.96\times std/\sqrt{100}$.}
\label{sw_tab_bt}
\end{table*}

In the privacy experiments, we compare the Mahalanobis mechanism with the multivariate Laplace mechanism on the following two privacy statistics:
\begin{enumerate}
    \item $N_w=\Pr\{M(w)=w\}$, which is the probability of the word not getting redacted in the mechanism. This is approximated by counting the number of times an input word $w$ does not get replaced by other words after running the mechanism for 100 times.
    \item $S_w=|\{w'\in\mathcal{W}: \Pr\{M(w)=w'\}\geq \eta|$,  which is the number of distinct words that have a probability greater than $\eta$ of being the output of $M(w)$. This is approximated by counting the number of distinct substitutions for an input word $w$ after running the mechanism for 100 times.
\end{enumerate}

We note that $N_w$ and $S_w$ have been previously used in privacy-preserving text analysis literature
to qualitatively characterize the privacy guarantee \cite{feyisetan2019leveraging,feyisetan2020privacy}.
We make the following connection between the privacy statistics ($N_w$ and $S_w$) with the DP privacy budget $\epsilon$:
a smaller $\epsilon$ corresponds to a stronger privacy guarantee by adding a larger scale ($1/\epsilon$) of noise in the mechanism, which leads to fewer unperturbed words (lower $N_w$) and more diverse outputs for each word (higher $S_w$).
For any fixed noise scale of $1/\epsilon$, the mechanism with a better privacy guarantee will have a lower value of $N_w$ and higher value of $S_w$. 

In Figure \ref{nw_ft} and \ref{sw_ft}, we summarize how the $5^{th}$, $50^{th}$, and $95^{th}$ percentiles of $N_w$ and $S_w$ change in different configurations of the mechanisms using the 300-d FastText embedding \cite{bojanowski2017enriching}. 
The vocabulary set includes 28,596 words in the vocabulary union from the two real datasets.
For all $5^{th}$, $50^{th}$, and $95^{th}$ percentiles, the Mahalanobis mechanism has a lower value of $N_w$ and a higher value of $S_w$ as compared to the multivariate Laplace mechanism, which indicates an improvement in privacy statistics. 

Table \ref{nw_tab} and \ref{sw_tab} compare the mean and standard deviation of $N_w$ and $S_w$ across different settings.
The mean $N_w$ converges to 0 as $\epsilon$ decreases and converges to 100 as $\epsilon$ increases. An opposite trend is observed for $S_w$, which is as expected. In the middle range of privacy budget ($\epsilon=5,10,20$), the proposed Mahalanobis mechanism has significantly lower values of $N_w$ and higher values of $S_w$, where the statistical significance is established by comparing the 95\% confidence intervals for mean $N_w$ and $S_w$ in the form of $mean$ $\pm$ $1.96\times std/\sqrt{100}$. 
While the scale of the noise is controlled to be the same across settings, the probability that a word does not change becomes smaller and the number of distinct substitutions becomes larger in our proposed mechanism. This shows the advantage of the Mahalanobis mechanism over the multivariate Laplace mechanism in privacy statistics.

The results are qualitatively similar when we repeat the same set of privacy experiments using the 300-d GloVe embeddings \cite{pennington2014glove}. As can be seen in Figure \ref{nw_bt}, Figure \ref{sw_bt}, Table \ref{nw_tab_bt}, and Table \ref{sw_tab_bt}), the Mahalanobis mechanism has lower values of $N_w$ and higher values of $S_w$ compared to the Laplace mechanism, which demonstrates a better privacy guarantee.

\begin{figure*}[h!]
    \centering
    \includegraphics[width=0.85\textwidth]{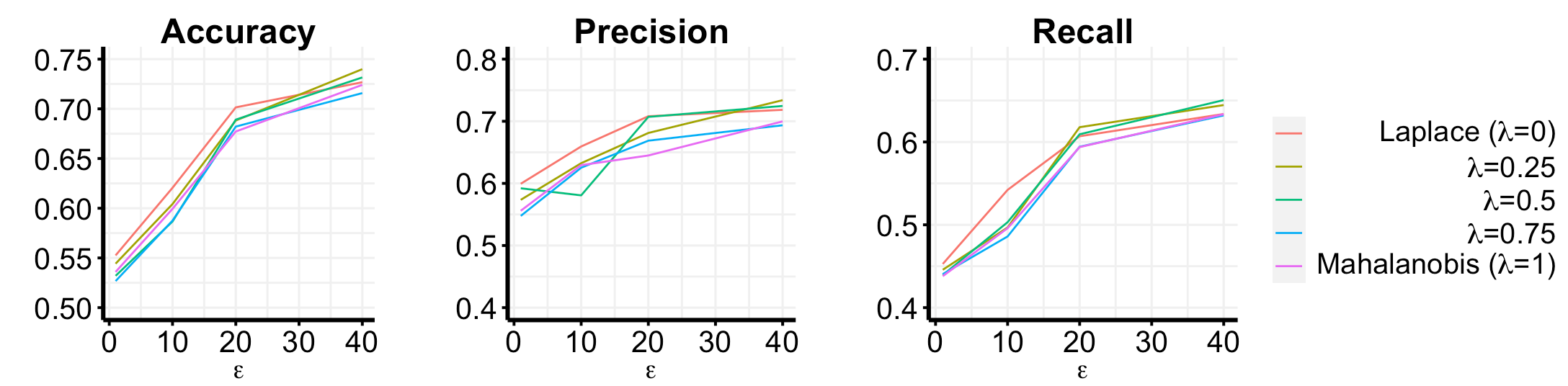}
    \caption{Text classification results on Twitter Dataset. There is no significant difference across mechanisms in terms of accuracy, precision, and recall. This shows the utility is maintained at the same level in the proposed Mahalanobis mechanism across $\lambda$.}
    \label{twitter}
\end{figure*}

\begin{figure*}[h!]
    \centering
    \includegraphics[width=0.85\textwidth]{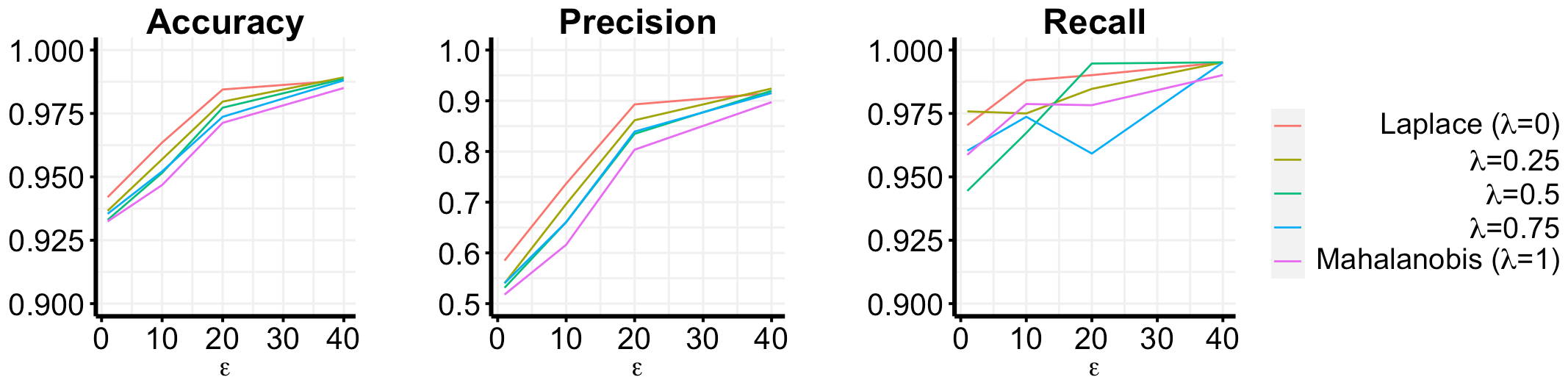}
    \caption{Text classification results on SMSSpam Dataset. There is no significant difference across mechanisms in terms of accuracy, precision, and recall. This shows the utility is maintained at the same level in the proposed Mahalanobis mechanism across $\lambda$.}
    \label{spam}
\end{figure*}

\subsection{Utility Experiments}

In the utility experiments, we compare the Mahalanobis mechanism with the multivariate Laplace mechanism in terms of text classification performance on the two real datasets. 

On the Twitter Dataset, the task is to classify whether a tweet describes a disaster event, where the benchmark FastText model \cite{joulin2016bag} achieves 0.78 accuracy, 0.78 precision, and 0.69 recall.
On the SMSSpam Dataset, the task is spam classification, where the
benchmark Bag-of-Words model achieves 0.99 accuracy, 0.92 precision, and 0.99 recall.   

In both tasks, we use 70\% of the data for training, and 30\% of the data for testing. The word embedding vectors are from 300-d FastText. 
Figure \ref{twitter} and Figure \ref{spam} present the utility results in terms of accuracy, precision and recall on two text classification tasks, respectively.
As a general trend in both Twitter and SMSSpam Dataset, the classification accuracy increases with $\epsilon$ and eventually approaches the benchmark performance, which is as expected. 
There are cases where the recall drops when $\epsilon$ increases in SMSSpam Dataset, but such drop is not significant as the recall values are all around 0.95 or higher.
Across the range of $\lambda$, the difference between utility is negligible between the Mahalanobis mechanism and the multivariate Laplace mechanism.
Together with results in Section 5.1, we conclude that our proposed mechanism improves the privacy statistics while maintaining the utility at the same level.

\section{Conclusions}

We develop a differentially private Mahalanobis mechanism for text perturbation. Compared to the existing multivariate Laplace mechanism, our mechanism exploits the geometric property of elliptical noise so as to improve the privacy statistics while maintaining a similar level of utility. 
Our method can be readily extended to the privacy-preserving analysis on other natural language processing tasks, where utility can be defined according to specific needs. 

We remark that the choice of $\Sigma$ as the global covariance matrix of the word embeddings can be generalized to the personalized covariance matrix within the neighborhood of each word. In this sense, local sensitivity can be used instead of global sensitivity to calibrate the privacy-utility tradeoff.
This can be done by adding a preprocessing clustering step on the word embeddings in the vocabulary, and then perform the Mahalanobis mechanism within each cluster using the cluster-specific covariance matrix. 

Furthermore, the choice of the tuning parameter $\lambda$ can also be formulated as an optimization problem with respect to pre-specified privacy and utility constraints. Since $\lambda$ is the only tuning parameter on a bounded interval of $[0,1]$, a grid search would suffice, which can be conducted by finding the $\lambda$ value that maximizes the utility (privacy) objective given the fixed privacy (utility) constraints.

\bibliographystyle{acl_natbib}
\bibliography{emnlp2020}
\appendix
\section{Additional Information on the Twitter and SMSSpam Data}\label{additional}

\begin{table}[h]
\begin{tabular}{lll} 
\hline
 Data & Twitter & SMSSpam \\ \hline
 Number of records & 7,613 & 5,574 \\
 Words per record & 7.3$\pm$3.0 & 15.6$\pm$11.4 \\
 Vocabulary size & 22,213 & 9,515 \\ 
 $d_{max} / d_{min}$ & 35.28 & 8.92 \\
 $\overline{d}_{:50} / \overline{d}_{-50:}$ & 4.47 & 5.67 \\
 $\overline{p}_{:50}$ & 0.1\% & 0.4\% \\
 \hline
\end{tabular}
\caption{$d_{max}$ is the largest Euclidean distance to the nearest word embedding in the vocabulary; $d_{min}$ is the smallest Euclidean distance to the nearest word embedding in the vocabulary; $\overline{d}_{:50}$ is the mean distance to the nearest word embedding for the top 50 words in the sparse regions (largest distance to nearest neighbors); $\overline{d}_{-50:}$ is the mean distance to the nearest word embedding for the top 50 words in the dense regions (smallest distance to nearest neighbors).
These ratio statistics demonstrate the heterogeneity in the density of the word embedding space. $\overline{p}_{:50}$ is the mean percentage of records that contains the top 50 words in the sparse regions, which suggests that those words in the sparse regions are indeed rare words that can be used to link to specific records.}
\label{twitter_smsspam}
\end{table}

\end{document}